\documentclass[a4paper,10.5pt]{article}
\usepackage{amssymb}
\usepackage[ruled,vlined]{algorithm2e}
\usepackage{amsthm}
\usepackage{amsfonts}
\usepackage{mathtools}
\usepackage{dsfont}
\usepackage{lastpage}

\usepackage{graphics}

\usepackage[a4paper,top=3cm,bottom=2cm,left=3.5cm,right=3.5cm,marginparwidth=2cm]{geometry}

\newtheorem{definition}{Definition}
\newtheorem{theorem}{Theorem}
\newtheorem{lemma}{Lemma}
\newtheorem{corollary}{Corollary}

\newcommand{\pbil}{\text{\sc PBIL}\xspace} 
\newcommand{\eda}{\text{\sc EDA}\xspace} 
\newcommand{\edas}{\text{\sc EDAs}\xspace} 
\newcommand{\cga}{\text{\sc cGA}\xspace} 
\newcommand{\umda}{\text{\sc UMDA}\xspace} 
\newcommand{\om}{\text{\sc OneMax}\xspace} 
\newcommand{\bval}{\text{\sc BinVal}\xspace} 
\newcommand{\los}{\text{\sc LeadingOnes}\xspace} 

\newcommand{\dkw}{\text{\sc DKW}\xspace} 

\newcommand{\Natural}{\mathbb{N}} 
\newcommand{\sspace}{\mathcal{X}}
\providecommand{\keywords}[1]{\textbf{\textit{Index terms---}} #1}

\newcommand{\prob}[1]{\Pr\left(#1\right)}
\newcommand{\expect}[1]{\mathbb{E}\left[#1\right]}

\newcommand{\bigO}[1]{\mathcal{O}\left(#1\right)} 
\newcommand{\bigOmega}[1]{\Omega\left(#1\right)} 
\newcommand{\bigTheta}[1]{\Theta\left(#1\right)} 
\newcommand{\step}[1]{\textbf{Step #1}:} 

\DeclarePairedDelimiter\ceil{\lceil}{\rceil}
\DeclarePairedDelimiter\floor{\lfloor}{\rfloor}

\usepackage[numbers]{natbib}
\title{Level-Based Analysis 
of the Population-Based Incremental Learning Algorithm\footnote{Preliminary 
version of this work will appear
in the Proceedings of the 15th 
International Conference on Parallel 
Problem Solving from Nature 2018 (PPSN XV).}
}
\author{Per Kristian Lehre \& Phan Trung Hai Nguyen\\
School of Computer Science\\
University of Birmingham\\
Birmingham B15 2TT, United Kingdom
}

\begin{document}
\maketitle

\begin{abstract}
The Population-Based Incremental Learning (\pbil) algorithm
uses a convex combination of 
the current model and the 
empirical model to construct the next model,
which is then sampled to generate offspring.
The Univariate Marginal Distribution Algorithm (\umda) 
is a special case of the \pbil, where the current model is ignored.
Dang and Lehre (GECCO 2015) 
showed that \umda can optimise 
\los efficiently. The question still remained open 
if the \pbil performs equally well. 
Here, by applying the level-based theorem 
in addition to  Dvoretzky--Kiefer--Wolfowitz inequality,
we show that the \pbil optimises 
function \los in expected time
$\bigO{n\lambda\log \lambda+n^2}$ for a population size 
$\lambda=\Omega(\log n)$,
which matches the bound of the \umda.
Finally, we show that the result 
carries over to \bval, giving the fist runtime
result for the \pbil on the \bval problem.
\end{abstract}

\keywords{Population-based incremental learning, LeadingOnes, BinVal,
Running time analysis, Level-based analysis, Theory}

\section{Introduction}

Estimation of distribution algorithms (\edas) 
are a class of randomised search heuristics
that optimise  objective functions 
by constructing probabilistic models and 
then sample the models to generate offspring 
for the next generation. 
Various variants of 
\eda have been proposed over the last decades; 
they differ from each other 
in the way their models are 
represented, updated as well as sampled over generations. 
In general, \edas are usually categorised
into two main classes: \textit{univariate} 
and \textit{multivariate}. Univariate \edas 
take advantage of first-order statistics (i.e. mean) to 
build a univariate model, whereas
multivariate \edas apply higher-order 
statistics to model the correlations
between the decision variables.

There are only a few runtime results available for \edas.
Recently, there has been a growing 
interest in the optimisation time
of the \umda, introduced by M{\"u}hlenbein and Paa{\ss}
\cite{Muhlenbein1996}, on standard benchmark functions
\cite{Dang:2015:SRA:2739480.2754814,bib:Witt2017,bib:Lehre2017,bib:Krejca,bib:Wu2017}.
Recall that the optimisation time of an algorithm is the number of fitness
evaluations the algorithm needs before a
global optimum is sampled for
the first time.
Dang and Lehre~\cite{Dang:2015:SRA:2739480.2754814} 
analysed a variant 
of the \umda using truncation selection and derived  
the first upper bounds of $\bigO{n\lambda \log\lambda}$ 
and $\bigO{n\lambda\log \lambda + n^2}$
on the expected optimisation times of the \umda on 
\om and \los, respectively,
where  the population size is $\lambda=\Omega(\log n)$.
These results were obtained using a relatively new 
technique called \textit{level-based analysis} \cite{bib:Corus2016}. 
Very recently, Witt~\cite{bib:Witt2017} proved that the 
\umda optimises \om within
$\bigO{\mu n}$ and $\bigO{\mu\sqrt{n}}$ when
$\mu \geq c\log n$ and $\mu \geq c'\sqrt{n}\log n$ for
some constants $c,~c'>0$, respectively. However, these bounds only hold 
when $\lambda = (1+\Theta(1))\mu$. This constraint on $\lambda$ and 
$\mu$ was relaxed by 
Lehre and Nguyen \cite{bib:Lehre2017}, where
the upper bound $\bigO{\lambda n}$ holds 
for $\lambda=\Omega(\mu)$ and $c\log n \leq \mu=\bigO{\sqrt{n}}$
for some constant $c>0$.

The first rigorous runtime 
analysis of the \pbil \cite{Baluja:1994:PIL:865123},
was presented very recently by Wu et al. \cite{bib:Wu2017}. 
In this work, the \pbil was referred to as a cross entropy algorithm.
The study proved an upper bound 
$\bigO{n^{2+\varepsilon}}$ of the \pbil with 
margins $[1/n,1-1/n]$ on \los, where
$\lambda=n^{1+\varepsilon}$, $\mu=\mathcal{O}(n^{\varepsilon/2})$,
$\eta\in \bigOmega{1}$ and $\varepsilon\in (0,1)$.
Until now, the known runtime bounds for the \pbil were significantly
higher than those for the \umda.  Thus, it is of interest to determine
whether the \pbil is less efficient than the \umda, or whether the bounds
derived in the early works were too loose.

This paper makes two contributions. First, we address the 
question above by deriving a tighter 
bound $\bigO{n\lambda\log \lambda+n^2}$
on the expected optimisation time of the \pbil on \los. The bound
holds for population sizes $\lambda=\bigOmega{\log n}$, 
which is a much weaker assumption than 
$\lambda=\omega(n)$ as required in \cite{bib:Wu2017}.
Our proof is more straightforward than that
in \cite{bib:Wu2017} because much of the
complexities of the analysis are already handled by
the level-based method \cite{bib:Corus2016}. 

The second contribution is the first runtime bound of the \pbil on
\bval. This function was shown to be the hardest among all linear
functions for the \cga \cite{bib:Droste2006}. The result carries
easily over from the level-based analysis of \los using an identical
partitioning of the search space. This observation further shows that
runtime bounds, derived by the level-based method using the canonical
partition, of the \pbil or other non-elitist population-based
algorithms using truncation selection, on \los also hold for
\bval.

The paper is structured as follows. 
Section~\ref{sec:preliminaries} introduces 
the \pbil with margins as well as the level-based theorem,
which is the main method employed in the paper. Given all 
necessary tools, the next two sections then
provide  upper bounds on the expected optimisation time of 
the \pbil  on \los and \bval. 
Finally, our concluding remarks are given in 
Section~\ref{sec:conclusion}.

\section{Preliminaries}
\label{sec:preliminaries}

We first introduce the notations used 
throughout the paper. Let
$\mathcal{X}:=\{0,1\}^n$ be a finite 
binary search space with dimension $n$.
The univariate model in generation $t\in \mathbb{N}$ 
is represented by a vector 
$p^{(t)}:=(p_1^{(t)},\ldots,p_n^{(t)})\in [0,1]^n$, 
where each $p_i^{(t)}$ is called a \textit{marginal}. Let 
$X_1^{(t)},\ldots,X_n^{(t)}$ be $n$ independent Bernoulli random variables 
with success probabilities $p_1^{(t)},\ldots,p_n^{(t)}$.
Furthermore, let $X_{i:j}^{(t)}:=\sum_{k=i}^j X_k^{(t)}$ 
be the number of ones
sampled from 
$p_{i:j}^{(t)}:=(p_i^{(t)},\ldots,p_j^{(t)})$ for 
all $1\le i\le j\le n$.
Each individual (or bitstring) is denoted as 
$x=(x_1,\ldots,x_n) \in \mathcal{X}$. 
We aim at maximising 
an objective function $f:\mathcal{X} \rightarrow \mathbb{R}$.
We are primarily interested in the optimisation time of these
algorithms, so tools to analyse runtime are of importance. We will
make use of the level-based theorem \cite{bib:Corus2016}. 

\subsection{Two problems}
We consider the two pseudo-Boolean functions:
\los and \bval, which are widely used 
theoretical benchmark problems in runtime
analyses of \edas \cite{bib:Droste2006,Dang:2015:SRA:2739480.2754814,bib:Wu2017}.
The former aims at maximising the number of leading ones, 
while the latter tries to maximise
the binary value of the bitstring. 
The global optimum for both functions are the all-ones bitstring.
Furthermore, \bval is an extreme linear function, 
where the fitness-contribution of the bits 
decreases exponentially with the bit-position.
Droste \cite{bib:Droste2006} 
showed that among all linear functions, \bval is difficult for the \cga.
Given a bitstring $x=(x_1,\ldots,x_n) \in \mathcal{X}$, the two functions 
are formally defined as follows:
\begin{definition}
$\los(x) := \sum_{i=1}^{n}\prod_{j=1}^{i}x_j$.
\end{definition}
\begin{definition}
$\bval(x) := \sum_{i=1}^{n}2^{n-i}x_i$.
\end{definition}

\subsection{Population-Based Incremental Learning}
\label{sec:pbil-algorithm}

The PBIL algorithm maintains a univariate model over generations. 
The probability of a bitstring $x=(x_1,\ldots,x_n)$ sampled 
from the current model $p^{(t)}$ is given by
\begin{equation}\label{eq:product-distribution}
\Pr\left(x\mid p^{(t)}\right)
=\prod_{i=1}^{n}\left(p_i^{(t)}\right)^{x_i} \left(1-p_i^{(t)}\right)^{1-x_i}.
\end{equation}
Let $p^{(0)}:=(1/2,\ldots,1/2)$ be the initial model. The algorithm 
in generation $t$ samples
a population of $\lambda$ individuals, denoted as
$P^{(t)}:=\{x^{(1)},x^{(2)},\ldots,x^{(\lambda)}\}$, which are 
sorted in descending order according to fitness. 
The $\mu$ fittest individuals are then selected to derive the
next model $p^{(t+1)}$ using the component-wise formula  
$p_i^{(t+1)} := \left(1-\eta\right) 
p_i^{(t)}+(\eta/\mu)\sum_{j=1}^{\mu}x^{(j)}_i$ 
for all $i\in \{1,2,\ldots,n\}$, 
where $x^{(j)}_i$ is the 
$i$-th bit of the $j$-th individual 
in the sorted population, and 
$\eta\in (0,1]$  is the smoothing parameter 
(sometimes known as the learning rate). 
The ratio $\gamma_0:=\mu/\lambda\in (0,1)$ is  
called the selective pressure of the algorithm.
Univariate EDAs often employ margins to avoid the
marginals to fix at either 0 or 1. In particular, the marginals are 
usually restricted to the interval $[1/n,1-1/n]$ after being updated,
where the quantities $1/n$ and $1-1/n$ are called 
the lower and upper borders, respectively.
The algorithm is called the \pbil with margins. 
Algorithm~\ref{pbil-algor} gives a full
description of the \pbil (with margins).

\begin{algorithm}[t]
	\DontPrintSemicolon		
		$t\leftarrow 0$; $p^{(t)}\leftarrow (1/2,1/2,\ldots,1/2)$\;
        \Repeat{termination condition is fulfilled}{
     	   \For{$j=1,2,\ldots,\lambda$}{
				sample an offspring $x^{(j)} \sim \Pr(\cdot\mid p^{(t)})$ 
                	as defined in (\ref{eq:product-distribution})\;
               	evaluate the fitness $f(x^{(j)})$\;
			}
            sort $P^{(t)}\leftarrow \{x^{(1)},x^{(2)},\ldots,x^{(\lambda)}\}$ such that
            $f(x^{(1)})\ge f(x^{(2)})\ge \ldots\ge f(x^{(\lambda)})$\;
            \For{$i=1,2,\ldots,n$}{
				$p_i^{(t+1)} \leftarrow \max\big\{1/n, \min\big\{1-1/n, 
                	\left(1-\eta\right) p_i^{(t)}
                    	+(\eta/\mu)\sum_{j=1}^{\mu}x_i^{(j)}\big\}\big\}$\;
          	}
            $t\leftarrow t+1$\;
		}
	\caption{\pbil with margins \label{pbil-algor}}
\end{algorithm}

\subsection{Level-based analysis} 

\begin{algorithm}[t]
	\DontPrintSemicolon
    $t \leftarrow 0$; create initial population $P^{(t)}$\;
    \Repeat{termination condition is fulfilled}{
      \For{$i=1,\ldots,\lambda$}{
				sample $P_i^{(t+1)} \sim \mathcal{D}(P^{(t)})$\;}
                $t \leftarrow t+1$\;
			}
	\caption{Non-elitist population-based algorithm\label{abstract-algor}}
\end{algorithm}

Introduced in \cite{bib:Corus2016}, 
the level-based theorem is a general tool that provides 
upper bounds on the expected optimisation time of 
many non-elitist population-based algorithms on a wide range of 
optimisation problems 
\cite{bib:Corus2016,bib:Lehre2017,Dang:2015:SRA:2739480.2754814}.
The theorem assumes that the algorithm to be 
analysed can be described in the form of 
Algorithm \ref{abstract-algor}, which maintains a population 
$P^{(t)}\in \sspace^{\lambda}$, 
where $\sspace^{\lambda}$ is the space 
of all populations with size $\lambda$. 
The theorem is general since it never assumes
specific fitness functions, selection mechanisms, or
generic operators like mutation and crossover. 
Furthermore, the theorem 
assumes that the search space $\mathcal{X}$ can be 
partitioned into $m$ disjoint subsets 
$A_1,\ldots,A_m$, which we call \textit{levels}, and 
the last level $A_m$ 
consists of all global optima of the objective function.
The theorem is formally stated  in 
Theorem~\ref{thm:levelbasedtheorem} \cite{bib:Corus2016}. 
We will use the notation
$[n]:=\{1,2,\ldots,n\}$ and
$A_{\ge j}:=\cup_{k=j}^m A_k$.

\begin{theorem}[\textsc{Level-Based Theorem}]\label{thm:levelbasedtheorem}
	Given a partition $\left(A_i\right)_{i \in [m]}$ of $\sspace$, define 
	\begin{math}
		T\coloneqq\min\{t\lambda \mid |P^{(t)}\cap A_m|>0\},
              \end{math}
              where for all $t\in\Natural$,
              $P^{(t)}\in\sspace^\lambda$ is the population of
              Algorithm~\ref{abstract-algor} in generation $t$.
              Denote $y \sim \mathcal{D}(P^{(t)})$. If there exist 
	$z_1,\ldots,z_{m-1}, \delta \in (0,1]$, and $\gamma_0\in (0,1)$ 
	such that for any population $P^{(t)} \in \sspace^\lambda$,
	\begin{itemize}
		\item[(\textbf{G1})] for each level 
        $j\in[m-1]$, 
        if $|P^{(t)}\cap A_{\geq j}|\geq \gamma_0\lambda$ then 
			$$
				\Pr\left(y \in A_{\geq j+1}\right) \geq z_j.
				$$
\item[(\textbf{G2})] for each level $j\in[m-2]$ and all 
       		 $\gamma \in (0,\gamma_0]$, if $|P^{(t)}\cap A_{\geq j}|\geq 	
        	\gamma_0\lambda$ and $|P^{(t)}\cap A_{\geq j+1}|\geq \gamma\lambda$ then
			$$
				\Pr\left(y \in A_{\geq j+1}\right) \geq \left(1+\delta\right)\gamma.
			$$
		\item[(\textbf{G3})] and the population size $\lambda \in \Natural$ satisfies
			$$
				\lambda \geq \left(\frac{4}{\gamma_0\delta^2}\right)\ln\left(\frac{128m}{z_*\delta^2}\right),
		$$
			where $z_* \coloneqq \min_{j\in [m-1]}\{z_j\}$, then
			\begin{displaymath}
				\mathbb{E}\left[T\right] \leq \left(\frac{8}{\delta^2}\right)\sum_{j=1}^{m-															1}\left[\lambda\ln\left(\frac{6\delta\lambda}{4+z_j\delta\lambda}\right)+\frac{1}{z_j}\right].
			\end{displaymath}
	\end{itemize}
\end{theorem}

Algorithm~\ref{abstract-algor} 
assumes a mapping 
$\mathcal{D}$ from the space of populations 
$\sspace^{\lambda}$ to the space of 
probability distributions over the search space. 
The mapping $\mathcal{D}$ is often said to depend 
on the current population only 
\cite{bib:Corus2016}; however, it is unnecessarily
always the case, especially for the \pbil
with a sufficiently large offspring population size $\lambda$. 
The rationale behind this is that in each generation 
the \pbil draws $\lambda$ samples from the current model
$p^{(t)}$, that correspond to $\lambda$ individuals 
in the current population, and 
if the number of samples $\lambda$ is 
sufficiently large, it is highly likely that 
the empirical distributions 
for all positions among the entire population 
cannot deviate too far from the true 
distributions, i.e. marginals $p_i^{(t)}$.
Moreover, the theorem relies on 
three conditions (G1), (G2) and (G3); thus, 
as long as these three can be fully verified, 
the \pbil, whose model is constructed from 
the current population $P^{(t)}$ in addition to 
the current model $p^{(t)}$, is 
still eligible to the level-based 
analysis. 

\subsection{Other tools}
In addition to the level-based theorem, we also 
make use of some other mathematical results. 
First of all is the Dvoretzky--Kiefer--Wolfowitz  inequality 
\cite{Massart-DKW}, which 
provides an estimate on how close an empirical 
distribution function will be to 
the true distribution from 
which the samples are drawn.
The following theorem follows 
by replacing $\varepsilon=\varepsilon'\sqrt{\lambda}$
into \cite[Corollary 1]{Massart-DKW}.

\begin{theorem}[\textsc{DKW Inequality}]
	\label{thm:dkw-inequality}
	Let $X_1,\ldots,X_{\lambda}$ be $\lambda$ i.i.d.
    real-valued random variables with 
	cumulative distribution function $F$. Let $\hat{F}_{\lambda}$ 
	be the empirical distribution function which is defined by
	$
	\hat{F}_{\lambda} (x) 
    := (1/\lambda)\sum_{i=1}^{\lambda} \mathds{1}_{\{X_i~\leq~ x\}}
	$. 	For any $\lambda\in \mathbb{N}$ and 
    $\varepsilon>0$, we always have
	$$
	\Pr\left(\sup_{x \in \mathbb{R}}\big| \hat{F}_{\lambda}(x)-F(x)\big| > \varepsilon\right) 
    \leq 2e^{-2\lambda\varepsilon^2}.
	$$	
\end{theorem}

Furthermore, properties of majorisation 
between two vectors 
are also exploited. The concept is formally defined 
in Definition~\ref{def:majorisation} \cite{gleser1975}, 
followed by its important property 
(in Lemma~\ref{boland-bound}) that we use 
intensively throughout the paper.

\begin{definition}
	\label{def:majorisation}
    Given vectors $p^{(1)}:=(p_1^{(1)},\ldots,p_n^{(1)})$ and
    $p^{(2)}:=(p_1^{(2)},\ldots,p_n^{(2)})$, 
    where $p_1^{(1)} \ge p_2^{(1)}\ge \ldots \ge p_n^{(1)}$ and similarly 
	for the $p_i^{(2)}$s. Vector $p^{(1)}$
	is said to \textit{majorise} vector $p^{(2)}$,    
	in symbols $p^{(1)} \succ p^{(2)}$, if  
	$p_1^{(1)} \ge p_1^{(2)}, \ldots, 
	\sum_{i=1}^{n-1} p_i^{(1)} \ge \sum_{i=1}^{n-1}p_i^{(2)}$ 
    and $\sum_{i=1}^{n} p_i^{(1)} = \sum_{i=1}^{n}p_i^{(2)}$.
	
\end{definition}

\begin{lemma}[\cite{Boland-1983}]\label{boland-bound}
Let $X_1,\ldots,X_n$ be $n$ independent Bernoulli random variables 
with success probabilities $p_1,\ldots,p_n$, respectively. 
Denote $p:= \left(p_1,p_2,\ldots,p_n\right)$; let
$S(p):= \sum_{i=1}^n X_i$ and
$D_{\lambda}:=\{p:p_i\in [0,1], ~i\in [n], ~
\sum_{i=1}^n p_i=\lambda \}$.
For two vectors $p^{(1)}, p^{(2)}\in D_{\lambda}$,
if $ p^{(1)} \prec p^{(2)}$ then
$\prob{S(p^{(1)})=n}\ge \prob{S(p^{(2)})=n}$.
\end{lemma}

\begin{lemma}\label{majorisation-after-update}
Let $p^{(1)}$ and $p^{(2)} \in D_{\lambda}$  be two vectors as defined in 
Lemma~\ref{boland-bound}, where all 
components in $p^{(\cdot)}$ are arranged in descending
order. Let $z^{(1)}:=(z_1^{(1)},\ldots,z_n^{(1)})$ where 
each $z_i^{(1)} := \left(1-\eta\right)p_i^{(1)}+\eta$,
and $z^{(2)}:=(z_1^{(2)},\ldots,z_n^{(2)})$,
where each $ z_i^{(2)}:=\left(1-\eta\right)p_i^{(2)}+\eta$ for 
any constant $\eta \in (0,1]$. 
If $p^{(2)} \succ p^{(1)}$, 
then $z^{(2)} \succ z^{(1)}$.
\end{lemma}
\begin{proof} 
For all $j\in [n-1]$, 
it holds that $\sum_{i=1}^j z_i^{(2)}\ge \sum_{i=1}^j z_i^{(1)}$
since $\sum_{i=1}^j p_i^{(2)} \ge \sum_{i=1}^j p_i^{(1)}$.
Furthermore, if $j=n$, then 
$\sum_{i=1}^n z_i^{(2)}=\sum_{i=1}^n z_i^{(1)}$ due to
$\sum_{i=1}^n p_i^{(2)}=\sum_{i=1}^n p_i^{(1)}$. 
By Definition~\ref{def:majorisation}, 
$z^{(2)} \succ z^{(1)}$. 
\end{proof}

\section{Runtime Analysis of the \pbil on \los}
\label{sec:pbil-leadingones}

We now show how to apply the level-based theorem to analyse the runtime of the \pbil.
We use a \textit{canonical partition} of the search space, where
each subset $A_j$ contains bitstrings with exactly 
$j$ leading ones. 
\begin{equation}\label{canonical-partition}
A_j:=\{x\in \{0,1\}^n\mid \los(x)=j\}.
\end{equation}
Conditions (G1) and (G2) of
Theorem~\ref{thm:levelbasedtheorem} assume that 
there are at least $\gamma_0\lambda$ individuals 
in levels $A_{\ge j}$ in generation $t$. 
Recall  $\gamma_0:=\mu/\lambda$. 
This implies that the first $j$ bits among the $\mu$ fittest individuals 
are all ones. Denote 
$\hat{p}_i^{(t)}:= (1/\lambda) \sum_{j=1}^{\lambda} x_i^{(j)}$
as the frequencies of ones at 
position $i$ in the current  population.
We first show that under the assumption 
of the two conditions of Theorem~\ref{thm:levelbasedtheorem} and
with a population size $\lambda=\bigOmega{\log n}$, 
the first $j$ marginals
cannot be too close to the lower border $1/n$
with probability at least $1-n^{-\Omega(1)}$.

\begin{lemma}\label{lem1}
If $|P^{(t)}\cap A_{\ge j}|\ge \gamma_0\lambda$ and 
$\lambda\ge
  c((1+1/\varepsilon)/\gamma_0)^2\ln(n)$ 
  for any constants $c,~\varepsilon>0$ and $\gamma_0\in(0,1)$, then
it holds with probability at least 
$1-2n^{-2c}$ 
that $p_i^{(t)}\ge \gamma_0/(1+\varepsilon)$ for all $i\in [j]$.
\end{lemma}
\begin{proof}
Consider an arbitrary bit $i\in [j]$. 
Let $Q_i$ be the number of 
ones sampled at position $i$ in the current 
population, and the corresponding 
empirical distribution function 
of the number of zeros
is $F_{\lambda}(0)
=(1/\lambda)\sum_{j=1}^{\lambda}\mathds{1}_{\{x_i^{(j)}\le 0\}}
=(\lambda-Q_i)/\lambda 
=1-\hat{p}_i^{(t)}$, and the true distribution function
is $F(0)=1-p_i^{(t)}$. The \dkw
inequality (see Theorem~\ref{thm:dkw-inequality})
yields that 
$\Pr(\hat{p}_i^{(t)}-p_i^{(t)}>\phi ) 
\le \Pr(|\hat{p}_i^{(t)}-p_i^{(t)}|>\phi) 
\le 2e^{-2\lambda\phi^2}$ for all $\phi>0$. 
Therefore, with probability at least $1-2e^{-2\lambda\phi^2}$ we have
$\hat{p}_i^{(t)}-p_i^{(t)}\le \phi$ and, thus, 
$p_i^{(t)}\ge \hat{p}_i^{(t)}-\phi \ge \gamma_0-\phi$ since
$\hat{p}_i^{(t)} \ge \gamma_0\lambda/\lambda = \gamma_0$ due to 
$|P^{(t)}\cap A_{\ge j}|\ge \gamma_0\lambda$.
We then choose $\phi \le \varepsilon\gamma_0/(1+\varepsilon)$ for some
constant $\varepsilon>0$ and $\lambda \ge 
c((1+1/\varepsilon)/\gamma_0)^2\ln(n)$. 
Putting everything together, it holds that 
$p_i^{(t)}\ge \gamma_0(1-\varepsilon/(1+\varepsilon))=\gamma_0/(1+\varepsilon)$
with probability at least $1-2n^{-2c}$.  
\end{proof}

Given the $\mu$ top individuals having at least $j$ leading ones,
we now estimate the probability of sampling $j$ leading ones 
from the current model $p^{(t)}$. 

\begin{lemma}\label{bound-on-prob-j-los}
  For any non-empty subset $I\subseteq[n]$, define
  $
    C_I := \big\{x\in \{0,1\}^n \mid \prod_{i\in I} x_i = 1\big\}.
  $
  If $|P^{(t)}\cap C_I|\ge \gamma_0\lambda$ and $\lambda\ge
  c((1+1/\varepsilon)/\gamma_0)^2\ln(n)$
  for any constants $\varepsilon>0,~\gamma_0\in(0,1)$,
  then it holds with
  probability at least $1-2n^{-2c}$
  that $q^{(t)}:=\prod_{i\in I} p_i^{(t)}\ge \gamma_0/(1+\varepsilon)$.
\end{lemma}
\begin{proof}
  We prove the statement using the \dkw inequality
  (see Theorem~\ref{thm:dkw-inequality}). Let $m=|I|$. 
  Given an offspring sample $Y\sim p^{(t)}$ from the current model, let $Y_I:=\sum_{i\in I} Y_i$ be the number of
  one-bits in bit-positions $I$. By the assumption $|P^{(t)}\cap C_I|\ge
  \gamma_0\lambda$ on the current population,
  the empirical distribution function of $Y_I$ must satisfy
$\hat{F}_{\lambda}(m-1)
=\frac{1}{\lambda}\sum_{i=1}^{\lambda}\mathds{1}_{\{Y_{I,i}\le m-1\}}
\leq 1-\hat{q}^{(t)}$, where 
$\hat{q}^{(t)}\geq\gamma_0$ is 
the fraction of individuals in the 
current population with $j$ leading ones, 
and the true distribution function satisfies
$F(m-1)=1-q^{(t)}$. The \dkw inequality 
yields that 
$\Pr(\hat{q}^{(t)}-q^{(t)}>\phi ) 
\le \Pr(|\hat{q}^{(t)}-q^{(t)}|>\phi) 
\le 2e^{-2\lambda\phi^2}$ for all $\phi>0$. 
Therefore, with probability at least $1-2e^{-2\lambda\phi^2}$ it holds
$\hat{q}^{(t)}-q^{(t)}\le \phi$ and, thus, 
$q^{(t)}\ge \hat{q}^{(t)}-\phi \ge \gamma_0-\phi$.
Choosing $\phi := \varepsilon\gamma_0/(1+\varepsilon)$, we get
$q^{(t)}\ge \gamma_0(1-\varepsilon/(1+\varepsilon))=\gamma_0/(1+\varepsilon)$
with probability at least $1-2e^{-2\phi^2\lambda} \geq 1-2n^{-2c}$.

\end{proof}

Given the current level is 
$j$, we speak of a \textit{success} if 
the first $j$ marginals never drop below $\gamma_0/(1+\varepsilon)$;
otherwise, we speak of a \textit{failure}. 
If there are no failures at all, let us assume that 
$\bigO{n\log \lambda+n^2/\lambda}$ is an upper bound
on the expected number of generations of the \pbil on \los.
The following lemma shows that this is also the the expected
optimisation time of the \pbil on \los.

\begin{lemma}\label{lem:expected-time}
	If the expected number of generations required by 
    the \pbil to optimise \los in case of no
	failure is at most $t^*\in \bigO{n\log \lambda+n^2/\lambda}$
    regardless of the initial probability vector of the \pbil,
    the expected number of generations of the \pbil on \los 
	is at most $4t^*$.
\end{lemma}

\begin{proof}
From the point when the algorithm starts, we divide the time into 
identical phases, each lasting $t^*$ generations.
Let $\mathcal{E}_i$ denote the event that the $i$-th
interval is a failure for $i\in \mathbb{N}$. 
According to Lemma~\ref{lem1}, 
$\prob{\mathcal{E}_i}
\le 2n^{-2c}~\mathcal{O}(n\log \lambda+n^2/\lambda)
=\mathcal{O}(n^{-c'+2})$ by union bound for another constant $c'>0$ when
the population is of at most exponential size, that is 
    $\lambda\le 2^{\alpha n}$ where $\alpha>0$ is a constant with
    respect to $n$, and the constant $c$ large enough
    such that $c'>2$,
and $\prob{\overline{\mathcal{E}}_1\wedge \overline{\mathcal{E}}_2}
\ge 1-\prob{\mathcal{E}_1} - \prob{\mathcal{E}_2} 
\ge 1-\mathcal{O}(n^{-c'+2})$ by union bound. Let $T$ be the 
number of generations performed by the algorithm 
until a global optimum is found for 
the first time. We know that
$\expect{T\mid \land_{i\in \mathbb{N}} ~\overline{\mathcal{E}}_i} \le t^*$, and 
$\prob{T\le 2t^*\mid \land_{i\in \mathbb{N}} ~\overline{\mathcal{E}}_i} \ge 1/2$ since 
$\prob{T\ge 2t^*\mid \land_{i\in \mathbb{N}} ~\overline{\mathcal{E}}_i} \le 1/2$
by Markov's inequality \cite{Mitzenmacher:2005:PCR:1076315}. 
We now consider each pair of two consecutive phases. 
If there is a failure in a pair of phases,  
we wait until that pair has passed by and then
repeat the arguments above as if no failure has ever happened.
It holds that
\begin{align*}
\expect{T\mid \overline{\mathcal{E}}_1 \wedge \overline{\mathcal{E}}_2} 
&\le 2t^*\prob{T\le 2t^*\mid \overline{\mathcal{E}}_1 \wedge \overline{\mathcal{E}}_2}
+ (2t^* + \expect{T})\prob{T\ge 2t^*\mid \overline{\mathcal{E}}_1\wedge
	\overline{\mathcal{E}}_2} \notag\\
&= 2t^* + \prob{T\ge 2t^*\mid \overline{\mathcal{E}}_1 \wedge \overline{\mathcal{E}}_2}\expect{T}\\
&\le 2t^* + (1/2)\expect{T}
\end{align*}
since 
$\prob{T\le 2t^*\mid \overline{\mathcal{E}}_1 \wedge \overline{\mathcal{E}}_2}\ge
\prob{T \le 2t^* \mid \wedge_{i\in \mathbb{N}} ~\overline{\mathcal{E}}_i}\ge 1/2$.
Substituting the result into the following 
yields
\begin{align*}
\expect{T} 
&= \prob{\overline{\mathcal{E}}_1 \wedge \overline{\mathcal{E}}_2}\expect{T\mid \overline{\mathcal{E}}_1 \wedge \overline{\mathcal{E}}_2}
+ \prob{\mathcal{E}_1 \vee \mathcal{E}_2}(2t^*+ \expect{T}) \\
&\le  \prob{\overline{\mathcal{E}}_1 \wedge \overline{\mathcal{E}}_2}(2t^*+(1/2)\expect{T}) + \prob{\mathcal{E}_1 \vee \mathcal{E}_2}(2t^*+ \expect{T})\\
&=2t^* + ((1/2)\prob{\overline{\mathcal{E}}_1 \wedge \overline{\mathcal{E}}_2} + \prob{\mathcal{E}_1\vee \mathcal{E}_2})\expect{T}\\
&=2t^* + \expect{T} -(1/2)\prob{\overline{\mathcal{E}}_1 \wedge \overline{\mathcal{E}}_2}\expect{T}. 
\end{align*}
Thus, $\expect{T} 
\le 4t^*/\prob{\overline{\mathcal{E}}_1\wedge \overline{\mathcal{E}}_2} 
= 4t^*\left(1+o(1)\right) = 4t^*$.
\end{proof}

By the result of Lemma~\ref{lem:expected-time}, 
the phase-based analysis that is exploited until there is 
a pair with no failure 
only leads to a multiplicative constant 
in the expectation. 
We need to calculate the value of $t^*$ that
will also asymptotically be the overall 
expected number of generations of the \pbil
on \los.  
We now give our runtime bound for the \pbil on \los 
with sufficiently large population $\lambda$.
The proof is very straightforward compared to that in \cite{bib:Wu2017}.
The floor and ceiling functions of 
$x\in \mathbb{R}$ are $\floor{x}$ and $\ceil{x}$, respectively.

\begin{theorem}\label{thm:pbil-on-los}
	The \pbil with margins and offspring population size
	$\lambda \ge c\log n$ for a sufficiently large constant $c>0$,
	parent population size $\mu =\gamma_0\lambda$ for any constant
    $\gamma_0$ satisfying 
    $\gamma_0 \le \eta^{\ceil{\xi}+1}/((1+\delta)e)$ 
    where $\xi = \ln (p_0)/(p_0-1)$ and $p_0:=\gamma_0/(1+\varepsilon)$  
    for any positive constants $\delta, ~\varepsilon$
    and smoothing parameter $\eta \in (0,1]$, 
	has expected optimisation time $\bigO{n\lambda\log \lambda +n^2}$
	on \los.
\end{theorem}

\begin{proof} 
We strictly follow the procedure recommended in \cite{bib:Corus2016}.
	
\step{1} Recall that we use the canonical partition,
defined in (\ref{canonical-partition}), in which 
each subset $A_j$ contains individuals with exactly $j$ leading ones.
There are a total of $m=n+1$ levels ranging from $A_0$ to $A_n$.

\step{2} 
 Given  $|P^{(t)}\cap A_{\ge j}|\ge \gamma_0\lambda=\mu$ and 
 $|P^{(t)}\cap A_{\ge j+1}|\ge \gamma\lambda$, we prove that the 
 probability of sampling an offspring in $A_{\ge j+1}$ 
 in generation $t+1$ is lower bounded by $(1+\delta)\gamma$ 
 for some constant $\delta>0$.  
 
Lemma~\ref{boland-bound} asserts that if we can find a vector 
$z^{(t)}=(z_{1}^{(t)},\ldots,z_{j}^{(t)})$ that majorises 
$p_{1:j}^{(t)}$, then the probability of obtaining $j$ successes from a 
Poisson-binomial distribution with parameters $j$ and $p_{1:j}^{(t)}$
is lower bounded by the same distribution with parameters $j$ and $z^{(t)}$.
Following \cite{bib:Wu2017}, we compare $X_1^{(t)},\ldots, X_j^{(t)}$
	with another sequence of independent Bernoulli random variables 
	$Z_1^{(t)},\ldots,Z_j^{(t)}$ with success probabilities 
	$z_1^{(t)},\ldots,z_j^{(t)}$. 
    Note that $Z^{(t)}:=\sum_{k=1}^j Z_k^{(t)}$.
	Define $m:=\floor{(\sum_{i=1}^j p_i^{(t)}-jp_0)/(1-\frac{1}{n}-p_0)}$ 
    where $p_0:=\frac{\gamma_0}{1+\varepsilon}$,
	and let $Z_1^{(t)},\ldots,Z_m^{(t)}$ all 
	have success probability $z_1^{(t)}=\ldots=z_m^{(t)}=1-\frac{1}{n}$, 
	$Z_{m+2}^{(t)},\ldots,Z_j^{(t)}$ get $p_0$ and possibly a 
	random variable $Z_{m+1}^{(t)}$ takes intermediate value 
	$[p_0, 1-\frac{1}{n}]$ to guarantee  
    $\sum_{i=1}^j p_i^{(t)}=\sum_{i=1}^{j} z_i^{(t)}$.

    Since 
    $\sum_{i=1}^{j}p_i^{(t)}
    \ge j\cdot (\prod_{i=1}^j p_i^{(t)})^{1/j}
    \ge j\cdot p_0^{1/j}$ by 
    the Arithmetic Mean-Geometric Mean inequality 
    (see Lemma~\ref{lem:am-gm inequality} in
    the Appendix) and 
    Lemma~\ref{bound-on-prob-j-los}, 
    we get 
    $m\ge \floor{j(p_0^{1/j}-p_0)/\left(1-\frac{1}{n}-p_0\right)}$. 
	Let us consider the following function: 
    $$
    g(j) = j\cdot \frac{p_0^{1/j}-p_0}{1-p_0} -j 
    = j\cdot \frac{p_0^{1/j}-1}{1-p_0}.
    $$
    This function has 
    a horizontal asymptote at $y=-\xi$, 
    where $\xi:=\frac{\ln p_0}{p_0-1}$ 
    (see calculation in the Appendix). 
    Thus, $m\ge j-\ceil{\xi}$ for all $j\ge 0$.  
    
    Note that we have just performed all calculations 
on the current model in generation $t$. 
The \pbil then updates the 
current model $p^{(t)}$ to obtain $p^{(t+1)}$ using 
the component-wise formula 
$p_i^{(t+1)}=(1-\eta)p_i^{(t)}+\frac{\eta}{\mu}\sum_{k=1}^{\mu}x_i^{(k)}$. 
For all $i\in [j]$, we know 
that $\sum_{k=1}^{\mu}x_i^{(k)}=\mu$ due to the assumption of condition (G2). 
After the model is updated, we obtain
\begin{itemize}
\item $z_i^{(t+1)} = 1-\frac{1}{n}$ 
for all $i\le j-\ceil{\xi}$,
\item $z_i^{(t+1)} \ge (1-\eta)~p_0+\eta\ge \eta$ for all $j-\ceil{\xi} < i\le j$, and
\item $p_{j+1}^{(t+1)}\ge (1-\eta)~p_{j+1}^{(t)}+\eta\frac{\gamma}{\gamma_0}
\ge \eta\frac{\gamma}{\gamma_0}$ 
due to $\sum_{k=1}^{\mu}x_{j+1}^{(k)}=\gamma\lambda$.
\end{itemize}
Let us denote $z_i^{(t+1)}= 
(1-\eta)z_i^{(t)}+\eta$. 
Lemmas~\ref{boland-bound} and 
\ref{majorisation-after-update} assert that 
$z^{(t+1)}$  majorises 
$p_{i:j}^{(t+1)}$, and 
$\Pr(X_{1:j}^{(t+1)}=j)\ge \Pr(Z^{(t+1)}=j)$. 
In words, the probability of 
sampling an offspring in 
$A_{\ge j}$ in generation $t+1$ is lower bounded 
by the probability of obtaining $j$ successes from a
Poisson-binomial distribution 
with parameters $j$ and $z^{(t+1)}$. More precisely, at generation $t+1$,
\begin{multline*}
\Pr(X_{1:j+1}^{(t+1)}=j+1)
\ge \Pr(X_{1:j}^{(t+1)}=j) \cdot \Pr(X_{j+1}^{(t+1)}=1)\\
\ge \Pr(Z^{(t+1)}=j)\cdot p_{j+1}^{(t+1)}\\
\ge (1-1/n)^{j-\ceil{\xi}} \eta^{\ceil{\xi}+1}\gamma/\gamma_0
\ge (1+\delta)\gamma,
\end{multline*}
where $\left(1-\frac{1}{n}\right)^{j-\ceil{\xi}} \ge \frac{1}{e}$ 
and $\gamma_0\le \frac{\eta^{\ceil{\xi}+1}}{e(1+\delta)}$ 
for any constant $\delta >0$. Thus, condition (G2) of 
Theorem~\ref{thm:levelbasedtheorem}
is verified.


\step{3} Given that $|P^{(t)}\cap A_{\ge j}|\ge \gamma_0\lambda$, we aim at showing
that the probability of sampling an offspring in $A_{\ge j+1}$
in generation $t+1$
is at least $z_j$. Note in particular that
Lemma~\ref{bound-on-prob-j-los} yields  
$\Pr(X_{1:j}^{(t+1)}=j)\ge \frac{\gamma_0}{1+\varepsilon}$. 
The probability
of sampling an offspring in $A_{\ge j+1}$ in generation $t+1$ 
is lower bounded by
\begin{displaymath}
	\Pr(X_{1:j}^{(t+1)}=j)\cdot \Pr(X_{j+1}^{(t+1)}=1)
	\ge \frac{\gamma_0}{1+\varepsilon}\cdot \frac{1}{n}
	=: z_j.
\end{displaymath}
where $\Pr(X_{j+1}^{(t+1)}=1)=p_{j+1}^{(t+1)}\ge \frac{1}{n}$. 
Therefore, condition (G1) of Theorem~\ref{thm:levelbasedtheorem}
is satisfied with $z_j=z_*= \frac{\gamma_0}{(1+\varepsilon)n}$.

\step{4} 
Condition (G3) of Theorem~\ref{thm:levelbasedtheorem}
requires a population size 
$\lambda=\bigOmega{\log n}$. This bound matches with the condition 
on $\lambda\ge c\log n$ for some sufficiently 
large constant $c>0$ from the  previous lemmas. 
Overall, $\lambda = \bigOmega{\log n}$.

\step{5} 
When $z_j=\frac{\gamma_0}{(1+\varepsilon)n}$ where 
$\gamma_0 \le \frac{\eta^{\ceil{\xi}+1}}{(1+\delta)e}$ and 
$\lambda \ge c\log n$
for some constants $\varepsilon >0$, 
$\eta\in (0,1]$ and sufficiently large $c>0$,
all conditions of Theorem~\ref{thm:levelbasedtheorem} are verified. 
Using that $\ln\left(\frac{6\delta \lambda}{4+\delta \lambda z_j}\right)
<\ln \left(\frac{3\delta\lambda}{2}\right)$
an upper bound on the expected optimisation time of the \pbil on \los  
is guaranteed as follows.
\begin{multline*}
\left(\frac{8}{\delta^2}\right)\sum_{j=0}^{n-1}
\left[\lambda\ln\left(\frac{3\delta\lambda}{2}\right)+\frac{1}{z_j}\right]
< \frac{8}{\delta^2} n\lambda\log \lambda + 
\frac{8(1+\varepsilon)}{\delta^2\gamma_0}n^2 + o\left(n^2\right)
\in  \bigO{n\lambda\log \lambda + n^2}.
\end{multline*}
Hence, the expected number of generations 
$t^*$ is $\bigO{n\log \lambda+\frac{n^2}{\lambda}}$ 
for a sufficiently large $\lambda$ 
in the case of no
failure and, thus, meets the 
assumption in Lemma~\ref{lem:expected-time}.
The expected optimisation time of the \pbil 
on \los is still asymptotically 
$\bigO{n\lambda\log \lambda +n^2}$.
This completes the proof.
\end{proof}

Our improved upper bound of $\bigO{n^2}$ on the optimisation time of
the \pbil with population size $\lambda = \bigTheta{\log n}$ on \los
is significantly better than the previous bound 
$\bigO{n^{2+\varepsilon}}$ from \cite{bib:Wu2017}.  
Our result is not only stronger, but the proof is much simpler
as most of the complexities of the population dynamics of the
algorithm is handled by Theorem~\ref{thm:levelbasedtheorem}
\cite{bib:Corus2016}. Furthermore, we also
provide specific values for the 
multiplicative constants, i.e. $\frac{32}{\delta^2}$ and 
$\frac{32(1+\varepsilon)}{\delta^2\gamma_0}$ for the terms 
$n\lambda\log \lambda$ and $n^2$, respectively 
(see Step 5 in Theorem~\ref{thm:pbil-on-los}).  
Moreover, 
the result also matches the runtime bound of 
the \umda on \los for 
a small population $\lambda=\bigTheta{\log n}$ 
\cite{Dang:2015:SRA:2739480.2754814}.

Note that Theorem~\ref{thm:pbil-on-los} requires some condition
on the selective pressure, that is
$\gamma_0 \le \frac{\eta^{\ceil{\xi}+1}}{(1+\delta)e}$ 
where $\xi = \frac{\ln p_0}{p_0-1}$ 
and $p_0:=\frac{\gamma_0}{1+\varepsilon}$  
for any positive constants $\delta, ~\varepsilon$
and smoothing parameter $\eta \in (0,1]$. Although for 
practical applications, we have to address
these constraints to find a suitable set of values for $\gamma_0$, 
this result here tells us that there exists some settings 
for the \pbil such that it can optimise \los 
within $\bigO{n\lambda\log \lambda+n^2}$ time in expectation.

\section{Runtime Analysis of the \pbil on \bval}
\label{sec:pbil-binval}
We first partition the search space into non-empty disjoint subsets $A_0,\ldots,A_n$.

\begin{lemma}\label{bval-los-similarity}
	Let us define the levels as  
	$$
	A_j :=\bigg\{ x\in \{0,1\}^n\bigg| \sum_{i=1}^j 2^{n-i} \le \bval(x) < \sum_{i=1}^{j+1} 2^{n-i}\bigg\},
	$$
	for $j \in [n]\cup\{0\}$, where $\sum_{i=1}^{0}2^{n-i}=0$.
	If a bitstring $x$ has exactly $j$ leading ones, i.e.
    $\los(x)=j$, then $x\in A_j$. 
\end{lemma}
\begin{proof}
	Consider a bitstring $x=1^j0\{0,1\}^{n-j-1}$. 
	The fitness  contribution of the first $j$ leading ones to $\bval(x)$ is
	$\sum_{i=1}^{j}2^{n-i}$. The $(j+1)$-th 
	bit has no contribution, while that of the last
	$n-j-1$ bits ranges from zero to 
	$\sum_{i=j+2}^{n}2^{n-i}=\sum_{i=0}^{n-j-2}2^{i} = 2^{n-j-1}-1$.
	So overall,
$	\sum_{i=1}^{j}2^{n-i} \le \bval(x) \le \sum_{i=1}^{j+1}2^{n-i}-1 < \sum_{i=1}^{j+1}2^{n-i}.
$
Hence, the bitstring $x$ belongs to level $A_j$. 
\end{proof}

In both problems, all that matters to determine 
the level of a  bitstring is the position of the first 0-bit 
when scanning from the most significant to the least significant bits.
Now consider two bitstrings in the same level for \bval, 
their rankings after the population is sorted are also determined 
by some other less significant bits; however, the proof of 
Theorem~\ref{thm:pbil-on-los} never takes these bits into account.
Thus, the following corollary yields the first 
upper bound on the expected optimisation time of the \pbil 
and the \umda (when $\eta=1$) for \bval. 

\begin{corollary}\label{thm:pbil-on-bval}
The \pbil with margins and offspring population size
	$\lambda \ge c\log n$ for a sufficiently large constant $c>0$,
	parent population size $\mu =\gamma_0\lambda$ for any constant
    $\gamma_0$ satisfying 
    $\gamma_0 \le \eta^{\ceil{\xi}+1}/((1+\delta)e)$ 
    where $\xi = \ln (p_0)/(p_0-1)$ and $p_0:=\gamma_0/(1+\varepsilon)$  
    for any positive constants $\delta, ~\varepsilon$
    and smoothing parameter $\eta \in (0,1]$, 
	has expected optimisation time $\bigO{n\lambda\log \lambda +n^2}$
	on \bval.
\end{corollary}

\section{Conclusions}
\label{sec:conclusion}
Runtime analyses of \edas are scarce. Motivated by this, we have derived 
an upper bound of $\bigO{n\lambda\log \lambda + n^2}$ on
the expected optimisation time 
of the \pbil on both \los and \bval
for a population size $\lambda=\bigOmega{\log n}$. 
The result improves upon the previously best-known bound 
$\bigO{n^{2+\varepsilon}}$ from \cite{bib:Wu2017}, and requires 
a much smaller population size $\lambda=\bigOmega{\log n}$, and
uses relatively straightforward arguments. 
We also presents the first upper 
bound on the expected optimisation time 
of the \pbil on \bval. 

Furthermore, our analysis demonstrates that the level-based theorem 
can yield runtime bounds for \edas 
whose models are updated using information gathered from
the current and previous generations. An additional aspect of our analysis 
is the use of the \dkw inequality to bound the true distribution by 
the empirical population sample when the number of samples is large enough. 
We expect these arguments will lead to new results in runtime 
analysis of evolutionary algorithms.

\section*{Appendix}
\begin{lemma}[\textsc{AM-GM Inequality} \cite{Steele:2004}]
\label{lem:am-gm inequality}
Let $x_1,\ldots,x_n$ be $n$ non-negative real numbers. It always holds that
$$\frac{x_1+x_2+\cdots+x_n}{n}  \ge \left(x_1\cdot x_2\cdots x_n\right)^{1/n},$$
and equality holds 
if and only if $x_1=x_2=\cdots=x_n$.
\end{lemma}

\begin{proof}[\textsc{Proof of horizontal asymptote}]
The function can be rewritten as 
$g(j)=\frac{1}{1-p_0}\cdot \frac{p_0^{1/j}-1}{1/j}$. 
Denote $t:=1/j$, we obtain 
$g(t)=\frac{1}{1-p_0}\cdot \frac{p_0^t-1}{t}$. 
Applying L'H{\^{o}}pital's rule yields:
$$
\lim_{j\to +\infty} g(j) = \lim_{t\to 0^{+}} g(t) 
= \frac{ \lim_{t\to 0^{+}} \left(p_0^t \ln p_0\right)}{1-p_0} 
=\frac{\ln p_0}{1-p_0}=-\frac{\ln p_0}{p_0-1}.
$$
\end{proof}

\end{document}